\documentclass[conference,twocolumn]{IEEEtran}

\usepackage{cite}
\usepackage{amsmath,amssymb,amsfonts}
\usepackage{algorithmic}
\usepackage{graphicx}
\ifCLASSOPTIONcompsoc
\usepackage[caption=false,font=normalsize,labelfont=sf,textfont=sf]{subfig}
\else
\usepackage[caption=false,font=footnotesize]{subfig}
\fi
\usepackage{textcomp}
\usepackage{xcolor}
\usepackage{braket}
\usepackage{csquotes}
\def\BibTeX{{\rm B\kern-.05em{\sc i\kern-.025em b}\kern-.08em
    T\kern-.1667em\lower.7ex\hbox{E}\kern-.125emX}}

\usepackage[]{import}
\usepackage[british]{babel}

\usepackage{booktabs}

\usepackage{units}

\usepackage{fancyvrb}
\fvset{fontsize=\normalsize}

\usepackage{multicol}

\usepackage{amsfonts}       
\usepackage{dsfont}
\usepackage{nicefrac}       
\usepackage{bm}
\usepackage[]{mathtools}
\usepackage{amsmath}
\usepackage{commath}
\usepackage[mathscr]{euscript}
\usepackage{cancel}
\usepackage{cases}

\usepackage{tikz}
\usetikzlibrary{positioning}
\usetikzlibrary{fit}
\usepackage{graphicx}
\usepackage{enumitem}
\usepackage{adjustbox}

\usepackage{hyperref}       
\usepackage{url}            
\usepackage[capitalise,noabbrev]{cleveref}       

\usepackage{doi}

\usepackage{color,soul}

\newtheorem{theorem}{Theorem}
\newtheorem{lemma}{Lemma}

\newtheorem{corollary}{Corollary}
\newtheorem{definition}{Definition}

\newtheorem{proof}{Proof}

\let\svlim\lim\def\lim{\svlim\limits}

\let\iff\Leftrightarrow
\let\epsilon\varepsilon

\newcommand\mytop[2]{\genfrac{}{}{0pt}{}{\hfil #1}{\hfill #2}}

\newcommand{\graph}{\mathcal{G}}
\newcommand{\grh}{\mathcal{H}}
\newcommand{\grP}{\mathcal{G}_\pr}
\newcommand{\grQ}{\mathcal{G}_\qr}

\newcommand{\gr}[1]{\graph_{#1}}
\newcommand{\fac}[1]{\Phi_{#1}}

\newcommand{\dmp}{\pr_{\grP}}
\newcommand{\dmq}{\qr_{\grQ}}

\newcommand{\partfacs}{\bm{\varphi}}
\newcommand{\partfac}{\varphi}

\newcommand{\sparts}{\ensuremath{ \bm{\mathcal{N}} }}

\newcommand{\spart}{N}

\newcommand{\bparts}{\ensuremath{ \bm{\mathcal{B}} }}

\newcommand{\settocli}{\kappa}
\newcommand{\settocligr}{\mathcal{K}}

\newcommand{\pr}{\mathbb{P}}
\newcommand{\qr}{\mathbb{Q}}

\newcommand{\prjt}{\mathbb{P}^\jt}

\newcommand{\xvars}{\bm{X}}
\newcommand{\zvars}{\bm{Z}}
\newcommand{\wvars}{\bm{W}}
\newcommand{\yvars}{\bm{Y}}
\newcommand{\wvals}{\bm{w}}
\newcommand{\xvals}{\bm{x}}
\newcommand{\yvals}{\bm{y}}
\newcommand{\zvals}{\bm{z}}
\newcommand{\dom}{\mathcal{X}}
\newcommand{\wdom}{\mathcal{W}}
\newcommand{\xdom}{\mathcal{X}}
\newcommand{\ydom}{\mathcal{Y}}
\newcommand{\zdom}{\mathcal{Z}}

\newcommand{\prygz}{\pr_{\yvars\mid\zvals}}
\newcommand{\qrygz}{\qr_{\yvars\mid\zvals}}

\newcommand{\prygZ}{\pr_{\yvars\mid\zvars}}
\newcommand{\qrygZ}{\qr_{\yvars\mid\zvars}}

\newcommand{\gru}{\mathcal{U}}
\newcommand{\mn}{\mathcal{M}}
\newcommand{\mnf}[1]{\mathcal{M}_{#1}=(\gr{#1}, \Phi_{#1})}
\newcommand{\conddiv}{D_{\AB}^{(\alpha,\beta)}(\prygZ\mid\mid\qrygZ)}
\newcommand{\conddivzero}{D_{\AB}^{(0,0)}(\prygZ\mid\mid\qrygZ)}

\newcommand{\sgr}{\Gamma}

\newcommand{\AB}{\textit{AB}}

\newcommand{\ab}{\alpha\beta}
\newcommand{\func}{\mathcal{F}}

\newcommand{\jt}{\mathcal{T}}

\newcommand{\bigO}{\mathcal{O}}

\newcommand{\tw}{\omega}
\newcommand{\twmax}{\omega_\text{max}}

\newcommand{\SP}{\textit{SP}}
\newcommand{\SPQ}{\textit{SPQ}}

\newcommand{\clis}{\bm{\mathcal{C}}}
\newcommand{\seps}{\bm{\mathcal{S}}}
\newcommand{\sep}{\mathcal{S}}
\newcommand{\cli}{\mathcal{C}}

\newcommand{\xvar}{X}

\DeclareMathOperator*{\argmax}{arg\,max}

\usepackage[nonumberlist,style=super,automake,symbols,abbreviations]{glossaries-extra}

\makeglossaries[abbreviations]

\glssetcategoryattribute{general}{nohyper}{true}
\glssetcategoryattribute{acronym}{nohyper}{true}

\setabbreviationstyle[acronym]{long-short}

\newglossaryentry{support}{
  name={support},
  description={Set of elements in $\xdom$, $S\subseteq\xdom$, that have a non-zero probability in the distribution $P$, $\forall \xvals\in S: P(\xvals)>0$}
}

\newglossaryentry{pgmpy}{
  name={\texttt{pgmpy}},
  description={Python library for graphical models~\cite{ankan2015}.}
}
\newglossaryentry{mcgo}{
  name={\texttt{mcgo}},
  description={Method to compute the \gls{kl} divergence between 2 \glspl{bn} of different structures in~\cite{moral2021}.}
}

\newglossaryentry{jt}{
  name={junction tree/forest},
  description={Junction tree/forest of a chordal graph.}
}

\newglossaryentry{spart}{
  name={$\sparts$-partition},
  description={Set of sets of vertices in a chordal graph that partitions both the maximal clique of the graph and some given subset of the vertices of the graph. See \cref{def:s-part}.}
}

\newglossaryentry{bpart}{
  name={$\bparts$-partition},
  description={A set containing sets in an \gls{spart} that has vertices associated with some variable in a given set of variables $\yvars$. See \cref{def:b-part}.}
}

\newglossaryentry{sgr}{
  name={$\sparts$-graph},
  description={Graph obtained from using function $\sgr$ defined in \cref{def:s-part-graph}.}
}

\newglossaryentry{mnp}{
  name={marginal $\sparts$-probability},
  plural={marginal $\sparts$-probabilities},
  description={Marginal probability over the variables of a vertex set from some $\sparts$-partition defined in \cref{def:marg-fac}.}
}

\newabbreviation{pmf}{pmf}{probability mass function}
\newabbreviation{rv}{rv}{random variable}

\newabbreviation{kl}{KL}{Kullback-Leibler}

\newabbreviation[category={graph}]{dg}{DG}{directed graph}
\newabbreviation[category={graph}]{ug}{UG}{undirected graph}
\newabbreviation[category={graph}]{dag}{DAG}{directed acyclic graph}

\newabbreviation{bn}{BN}{Bayesian network}
\newabbreviation{mn}{MN}{Markov network}
\newabbreviation{dm}{DM}{decomposable model}
\newabbreviation{cpt}{CPT}{conditional probability table}
\newabbreviation{jta}{JTA}{junction tree algorithm}

\newabbreviation{comp}{MGASP}{multi-graph aggregated sum-product}
\newabbreviation{generator}{DMDG}{decomposable model drift generator}
\newabbreviation{method}{DMDE}{decomposable model divergence estimator}
\newabbreviation{JTC}{JTComp}{Junction Tree Computation}
\newabbreviation{JFC}{JFComp}{Junction Forest Computation}

\glsdisablehyper
\newcommand{\codelink}{\url{https://lklee.dev/pub/2023-icdm/code}}

\crefname{equation}{}{}

\usepackage{yquant}

\begin{document}

\title{Computing Marginal and Conditional Divergences between
  Decomposable Models with Applications
}

\author{\IEEEauthorblockN{{Loong Kuan
      Lee}\IEEEauthorrefmark{1},
{Geoffrey I. Webb}\IEEEauthorrefmark{2},
{Daniel F. Schmidt}\IEEEauthorrefmark{2}, 
{Nico Piatkowski}\IEEEauthorrefmark{1}}
\IEEEauthorblockA{\IEEEauthorrefmark{1}Fraunhofer IAIS, Schloss
  Birlinghoven, 53757 Sankt Augustin, Germany
}
\IEEEauthorblockA{
  \{loong.kuan.lee, nico.piatkowski\}@iais.fraunhofer.de
}
\IEEEauthorblockA{\IEEEauthorrefmark{2}Department of Data Science and
  AI, Monash University, Melbourne, Australia
}
\IEEEauthorblockA{
  \{geoff.webb, daniel.schmidt\}@monash.edu
}
}


\maketitle
\IEEEpeerreviewmaketitle

\begin{abstract}
  The ability to compute the exact divergence between two
  high-dimensional distributions is useful in many applications but
  doing so naively is intractable. 
  Computing the alpha-beta divergence---a family of divergences that
  includes the Kullback-Leibler divergence and Hellinger
  distance---between the joint distribution of two decomposable
  models,~i.e chordal Markov networks, can be done in time exponential
  in the treewidth of these models. However, reducing the dissimilarity
  between two high-dimensional objects to a single scalar value can be
  uninformative. Furthermore, in applications such as supervised
  learning, the divergence over a conditional distribution might be of
  more interest. Therefore, we propose an approach to compute the exact
  alpha-beta divergence between any marginal or conditional distribution
  of two decomposable models. Doing so tractably is non-trivial as we
  need to decompose the divergence between these distributions and
  therefore, require a decomposition over the marginal and conditional
  distributions of these models. Consequently, we provide such a
  decomposition and also extend existing work to compute the marginal
  and conditional alpha-beta divergence between these decompositions. We
  then show how our method can be used to analyze distributional changes
  by first applying it to a benchmark image dataset. Finally, based on
  our framework, we propose a novel way to quantify the error in
  contemporary superconducting quantum computers.  Code for all
  experiments is available at: \codelink
\end{abstract}


\section{Introduction}
The ability to analyze and quantify the differences between two
high-dimensional distributions has many applications in the fields of
machine learning and data science. For instance, 
in the
study of problems with changing distributions, i.e. concept
drift~\cite{schlimmer1986,webb2018}, in the detection of anomalous
regions in spatio-temporal data~\cite{barz2019,piatkowski2013}, and in
tasks related to the retrieval, classification, and visualization of
time series data~\cite{chen2020}.

There has been some previous work done on tractably computing the
divergence between two high-dimensional distributions modeled
by 
chordal \glspl{mn}, i.e. \glspl{dm}~\cite{lee2023}. However, reducing
the differences between two distributions to just one scalar value can
be quite uninformative as it does not help us better understand why the
distributions are different in the first place. Specifically in
high-dimensions, it is quite likely that the difference between two
distributions will be largely attributable to the influence of a small
proportion of the variables~\cite{webb2018}.

One way to better understand distributional changes in
higher-dimensions is to measure the divergence between the marginal
univariate distributions over individual variables in the two
distributions~\cite{webb2018}. However, this does not capture changes
in the relationships between variables~\cite{hinder2020}. We will
  see further evidence on the importance of measuring changes over
  multiple variables at once
  in
\cref{sec:quantum-single} and \cref{fig:errors}.

Therefore, the goal of this work is to develop a method for computing
the divergence between two \glspl{dm} over arbitrary variable subsets.
Doing so tractably is non-trivial, as computing the $\ab$-divergence
  involves a sum over an exponential number of elements with respect to
  the number of variables in the distributions.
Therefore we need a ``decomposition'' of the sum in the
  $\ab$-divergence into a product of smaller---more
  tractable---sums. As we will see later on, this will require a
  decomposition over any marginal distribution of a \gls{dm} into
  product of smaller marginal distributions as well. An example of
  such a decomposition can be found in \cref{fig:overview}.
%
\begin{figure}
  \centering
  \begin{adjustbox}{width=0.45\textwidth}
    \begin{tikzpicture}[vert/.style={circle, draw=black}]
        \node[vert] (1) at (-1,0.5) {$1$};
        \node[vert] (2) at (-1,-0.5) {$2$};
        \node[vert] (3) at (0,0) {$\cancel{3}$};
        \node[vert] (4) at (1.25,0) {$4$};
        \node[vert] (5) at (2.25,0.5) {$5$};
        \node[vert] (6) at (2.25,-0.5) {$6$};
        \node[] (b) at (3,0) {};
        \node[] (bEq) at (0.75,-1.5) {
            $\begin{gathered}
              \sum_{\xvals_{3}\in\xdom_{3}}
              \pr_{1,2,3,4,5,6}(\xvals_{3})
            \end{gathered}$
        };

        \node[] (a) at (4.5,0) {};
        \node[vert] (1a) at (5,0.5) {$1$};
        \node[vert] (2a) at (5,-0.5) {$2$};
        \node[vert] (4a) at (6,0) {$4$};
        \node[vert] (5a) at (7,0.5) {$5$};
        \node[vert] (6a) at (7,-0.5) {$6$};
        \node[] (aEq) at (6,-1.5) {
            $\begin{gathered}
              \pr_{1,2,4}\cdot\pr_{4,5}\cdot\pr_{4,6}
            \end{gathered}$
        };

        \draw[->, line width=1mm] (b) -- (a);
        \draw (1) -- (3);
        \draw (2) -- (3);
        \draw (4) -- (3);
        \draw (4) -- (5);
        \draw (4) -- (6);

        \draw (1a) -- (4a);
        \draw (2a) -- (4a);
        \draw (2a) -- (1a);
        \draw (4a) -- (5a);
        \draw (4a) -- (6a);
      \end{tikzpicture}       
    \end{adjustbox}\label{fig:n-part-gr}
    \caption{Example of decomposing the marginal distribution of a
      \gls{dm} into a product of maximal clique probabilities of a new
      \gls{dm}. Doing so for the general case is one of our main
      contributions. More details in
      \cref{sec:marg-div}.}\label{fig:overview}
\end{figure}
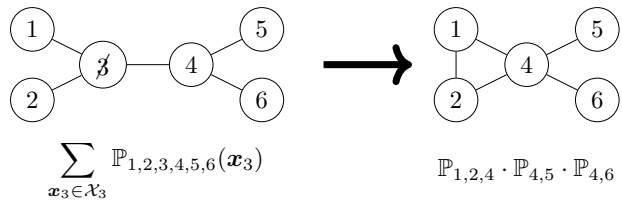

To this end, we will develop the tools needed to find such a
decomposition and also show that these tools can also
be used to develop a method for computing the $\ab$-divergence between
conditional distributions of \glspl{dm} as well. More generally,
  our method can be used to compute any functional that can be
  expressed as a linear combination of the Functional $\func$ in
\cref{def:func}, between any marginal or conditional distributions
  of 2  \glspl{dm}. In fact, it is the properties of $\func$, and
  the $\ab$-divergence being a linear combination of $\func$ for most
  values of $\alpha$ and $\beta$, that allows for the decomposition of
  the $\ab$-divergence over any distributions between two \glspl{dm}.

In order to underpin the relevance of our work, we provide a completely
novel error analysis of real-world quantum computing devices, based on
our methodology. It allows us to gain insights into the specific errors
of (subsets of) qubits, which is of utmost importance for practical
quantum computing.

\section{Related Work}
There is not much previous work on computing the marginal or
conditional divergences between probabilistic graphical models in
general, let alone the marginal and conditional $\ab$-divergence.
That said, the idea of isolating distributional changes to some
subspace of the distribution is not new. For instance,
\cite{hoens2011} employed multiple base learners, each trained over a
random subset of variables. When a change occurs, this approach can
either update the affected learners or reduce how much they will weigh
in on the final output of the ensemble.
In \cite{alberghini2022}, the authors create an ensemble of AESAKNNS
models~\cite{roseberry2021}---a model already capable of handling
distributional change on its own---on different subsets of variables and
samples. Furthermore, unlike \cite{hoens2011}, they vary the size of the
variable subsets for model training as well, and can dynamically change
it over time.

In order to explicitly detect subspaces where distributional changes
occur, \cite{liu2017} compute a statistic between the subspace of two
samples, called the \emph{local drift degree}, which follows the normal
distribution when both samples follow the same distribution. More
recently, \cite{polo2022} propose their own statistic, based on the
Kullback-Leibler divergence, for detecting changes in the distribution
over the covariate and class variables, as well as conditional
distributions between them.

\section{Background and Notation}
For any set of $n$ random variables,
$\xvars=\{\xvar_{1},\ldots,\xvar_{n}\}$, let
$\dom_{\xvars} = \dom_{\xvar_{1}} \times \ldots \times \dom_{\xvar_{n}}$
be the domain of $\xvars$, i.e. the set of values that $\xvars$ can
take, and $\xvals\in\dom_{\xvars}$ be one of these values.
Additionally, for some subset of variables $\zvars\subseteq\xvars$, let
$\bar{\zvars}$ be the set of variables in $\xvars$ not in $\zvars$,
$\bar{\zvars} = \xvars \setminus \zvars$.

\subsection{Statistical Divergences}
A divergence is a measure of the ``difference'' between two probability
distributions. More formally, a divergence is a functional
with the following properties:
\begin{definition}[Divergence]\label{def:divergence}
  Suppose $P$ is the set of probability distributions with the same
  support. A divergence, $D$, is the function
  $D(\cdot\ ||\ \cdot) : P \times P \rightarrow \mathbb{R}$ such that
  $\forall \pr,\qr \in S : D(\pr\ ||\ \qr) \geq 0$ and
  $\pr = \qr \iff D(\pr\ ||\ \qr) = 0$~\footnote{Some authors also
    require that the quadratic part of the Taylor expansion of
    $D(p, p + dp)$ define a Riemannian metric on $P$~\cite{amari2016}.
  }.
\end{definition}
Some popular divergences include the Kullback-Leibler (KL)
divergence~\cite{kullback1951}, the Hellinger
distance~\cite{hellinger1909}, and the Itakura-Saito (IS) divergence
\cite{itakura1968,fevotte2009}. Furthermore, these aforementioned
divergences can be expressed by a more general family of divergences,
known as the $\ab$-divergence.
\begin{definition}[$\alpha\beta$-divergence~\cite{cichocki2010}]
  \label{def:ab-divergence}
  The $\alpha\beta$-divergence, $D_{\AB}$, between two positive measures
  $\pr$ and $\qr$ is defined as
  \begin{gather}
    D_{\AB}^{\alpha,\beta}(\pr,\qr)
    = \sum_{\xvals\in\dom}d_{\AB}^{\alpha,\beta}
    \big(\pr(\xvals),\qr(\xvals)\big)\;,
  \end{gather}
  where $\alpha$ and $\beta$ are parameters, and
  \begin{align}
    &d_{\AB}^{(\alpha,\beta)}(\pr(\xvals), \qr(\xvals)) \label{eq:ab-div-ext}\\
    &=
      \begin{cases}
        \frac{-1
        }{\alpha\beta}\left(
        \pr(\xvals)^{\alpha}\qr(\xvals)^{\beta} -
        \frac{\alpha\pr(\xvals)^{\alpha + \beta} }{\alpha + \beta} -
        \frac{\beta\qr(\xvals)^{\alpha + \beta}}{\alpha + \beta}
        \right),
        & \hfill \mytop{\hfill\alpha,\beta\neq0}{\alpha+\beta\neq0}\\ 
        \frac{1}{\alpha^{2}}\left(
        \pr(\xvals)^{\alpha} \log
        \frac{\pr(\xvals)^{\alpha}}{\qr(\xvals)^{\alpha}} -
          \pr(\xvals)^{\alpha} + \qr(\xvals)^{\alpha}
        \right),& \hfill\mytop{\alpha\neq0}{\beta=0}\\
        \frac{1}{\alpha^{2}} \left(
          \log \frac{\qr(\xvals)^{\alpha}}{\pr(\xvals)^{\alpha}} +
        \left(\frac{\qr(\xvals)^{\alpha}}
        {\pr(\xvals)^{\alpha}}\right)^{-1} - 1
        \right),& \hfill\mytop{\alpha=-\beta}{\alpha\neq0}\\
        \frac{1}{\beta^{2}}\left(
        \qr(\xvals)^{\beta}\log
        \frac{\qr(\xvals)^{\beta}}{\pr(\xvals)^{\beta}} -
          \qr(\xvals)^{\beta} + \pr(\xvals)^{\beta}
        \right), & \hfill\mytop{\alpha=0}{\beta\neq 0}\\
        \frac{1}{2}(\log \pr(\xvals) - \log \qr(\xvals))^{2},
        & \hfill \alpha,\beta=0\;.
      \end{cases}\nonumber
    \end{align}
\end{definition}

To make subsequent developments easier, the $\ab$-divergence, when
either $\alpha$ or $\beta$ is $0$, can be expressed in terms of a linear
combination of the functional $\func$ in \cref{def:func}~\cite{lee2023}.
\begin{definition}[$\func$ Functional]\label{def:func}
  For any function $L$ with the property
  $L\left(\prod_{x}x\right)=\sum_{x}L(x)$, and functionals
  $f \in \{g, h, g^{*}, h^{*}\}$ with the property that
  $f\bigl[\prod_{\zvars \in \mathcal{A}} P_{\zvars}\bigr] =\prod_{\zvars
    \in \mathcal{A}} f\bigl[P_{\zvars}\bigr]$, where
  $\mathcal{A} \subset \mathcal{P}(\xvars)$ and
  $P_{\xvars}=\prod_{\zvars \in \mathcal{A}}P_{\zvars}$, let:
  \begin{align*}
    &\func(\pr,\qr; g, h, g^{*}, h^{*})\\
    &= \sum_{\xvals \in \dom}
    g\big[\pr\big](\xvals) \cdot
    h\big[\qr\big](\xvals) \cdot
    L\Bigl(
    g^{*}[\pr](\xvals) \cdot
    h^{*}[\qr](\xvals)\Bigr)
  \end{align*}
\end{definition}
Therefore, the $\ab$-divergence can be expressed as follows:
\begin{align}
  \begin{aligned}
    &D_{\AB}^{\alpha,\beta}(\pr,\qr)\\
    &=\begin{cases}
        \sum_{\xvals\in\xdom}
        \frac{1}{2}\left(\log \pr(\xvals) - \log \qr(\xvals)\right)^{2}
        ,& \alpha,\beta=0\\
        \sum_{i} c_{i} \func(\pr, \qr ; g_{i},h_{i},g^{*}_{i},h^{*}_{i})
        ,&\text{otherwise}
      \end{cases}
  \end{aligned}\label{eq:ab-div-f-expr}
\end{align}
which is more terse compared to \cref{eq:ab-div-ext}.

So far we have discussed divergences between joint distributions, but
computing these divergences between two marginal distributions,
$\pr_{\zvars}$ and $\qr_{\zvars}$, instead, is similar as we can just
plug these marginal distributions into any divergence directly.



Unfortunately, defining the divergence between two conditional
distributions, $\pr_{\yvars\mid\zvars}$ and $\qr_{\yvars\mid\zvars}$,
is not as straightforward as they are comprised of multiple
distributions over $\yvars$, one of each value of $\zvals\in\zdom$.
There is however a ``natural'' definition for the conditional
KL-divergence as the two natural approaches to define it,
lead to the same result~\cite{bleuler2020}:
\begin{align}
  D(\pr_{\yvars\mid\zvars}\mid\mid\qr_{\yvars\mid\zvars}\mid \pr_{\zvars})
  &:=D(\pr_{\yvars\mid\zvars} \pr_{\zvars}\mid\mid
  \qr_{\yvars\mid\zvars} \pr_{\zvars})\label{eq:kl-cond-1}\\
  &:= \mathbb{E}_{\zvars \sim \pr}
    \left[D(\pr_{\yvars | \zvars} \mid\mid \qr_{\yvars \mid \zvars})\right]
  \label{eq:kl-cond-2}
\end{align}
Unfortunately, this equivalence is unique to the KL-divergence.  For
defining other conditional divergences, previous approaches have either
exclusively taken the definition in
\cref{eq:kl-cond-1}~\cite{sibson1969,cai2019,bleuler2020} or
\cref{eq:kl-cond-2}~\cite{csiszar1995,poczos2012,bhattacharyya2018}. For
our purposes, we will use the definition in \cref{eq:kl-cond-2} which
involves taking the expectation of
$D(\pr_{\yvars\mid\zvals},\qr_{\yvars\mid\zvals})$ with respect to
$\pr_{\zvars}$.
\begin{definition}[Conditional Divergence]
  \label{def:cond-div}
  \begin{align*}
    D(\pr_{\yvars \mid \zvars} \mid\mid \qr_{\yvars \mid \zvars})
    &= \sum_{\zvals\in\zdom} \pr_{\zvars}(\zvals)
      D(\pr_{\yvars | \zvals} \mid\mid \qr_{\yvars \mid \zvals})
  \end{align*}
\end{definition}
\subsection{Graphical Models}
An undirected graph, $\gr{} = (V(\gr{}), E(\gr{}))$, consists of
$n=\abs{V(\gr{})}$ vertices, and edges $(u,v)\in E(\gr{})$. When a
subset of vertices, $\cli\subseteq V(\gr{})$, are fully connected, they
form a clique. Furthermore, let $\clis(\gr{})$ denote the set of all
maximal cliques in $\gr{}$, where a maximal clique is a clique that is
not contained in any other clique.
\subsubsection{Markov Network (MN)}
By associating each vertex in $\gr{}$ to a random variable,
$\xvars= \{\xvar_{v} \mid v \in V(\gr{})\}$, the graph $\gr{}$ will
encode a set of conditional independencies between the variables
$\xvars$~\cite{lauritzen1996}. For some probability distribution $\pr$
over the variables $\xvars$, $\pr$ is strictly positive and follows the
conditional independencies in $\gr{}$ if and only if it is a Gibbs
distribution of the form~\cite{hammersley1971}:
$\pr_{\xvars}(\xvals)=\frac{1}{Z}
\prod_{\cli\in\clis(\gr{})}\psi_{\cli}(\xvals_{\cli})$, where $Z$ is a
normalizing constant,
$Z=\sum_{\xvals\in\xdom}\prod_{\cli\in\clis(\gr{})}
\psi_{\cli}(\xvals_{\cli})$, and $\psi_{\cli}$ are positive factors over
the domain of each maximal clique,
$\psi_{\cli} : \dom_{\cli} \to \mathbb{R}$. We denote such a \gls{mn}
with the notation $\mn=(\gr{},\Psi)$, where
$\Psi=\{\psi_{\cli} \mid \cli\in\clis(\gr{})\}$. Markov networks are
also known as Markov random fields.
\subsubsection{Belief Propagation}
Unfortunately, computing $Z$ directly is
intractable due to the sum growing exponentially with respect to
$\abs{\xvars}$. However, by finding a chordal graph $\grh$ that is
strictly larger than $\gr{}$, we can use an algorithm known as
\emph{belief propagation} on the clique tree of $\grh$ with initial
factors $\Psi$, to compute not only $Z$, but also the marginal
distributions of $\pr$ over each $\cli\in\clis(\grh)$. The complexity of
belief propagation is $\bigO(2^{\tw(\grh)})$, where $\tw$ is the
treewidth of $\grh$, i.e. the size of the largest maximal clique in
$\grh$ minus one.
\subsubsection{Decomposable Model (DM)}
When a \gls{mn} has a chordal graph structure, it also has a closed form
joint distribution:
\begin{equation}
  \label{eq:dm-mle}
  \pr_{\gr{}}(\xvals)
  = \frac{\prod_{\cli\in\clis(\gr{})}
    \pr_{\cli}(\xvals_{\cli})}{\prod_{\sep\in\seps(\gr{})}
    \pr_{\sep}(\xvals_{\sep})}
  =\prod_{\cli\in\clis(\gr{})}\pr^{\jt}_{\cli}(\xvals_{\cli})
\end{equation}
where $\seps(\gr{})$ are the sets of variables between each adjacent
maximal clique in the clique tree of $\gr{}$, and $\pr^{\jt}_{\cli}$ is a
\gls{cpt} over clique $\cli$.

For clarity of exposition, we will only use \gls{dm} to refer to chordal
\glspl{mn} with the \glspl{cpt} in \cref{eq:dm-mle} as factors.
\section{Previous Work: Computing Joint Divergence}
\label{sec:joint-div}
It has been shown~\cite{lee2023} that the complexity of computing the
$\ab$-divergence between decomposable models $\dmp$ and $\dmq$ is
generally, in the worst case, exponential to the treewidth of the
computation graph $\gru(\grP,\grQ)$, as defined in
\cref{def:comp-graph}.
\begin{definition}[Computation graph, $\gru$]\label{def:comp-graph}
  Let $\gr{(1)}$ and $\gr{(2)}$ be two chordal graphs. Then denote
  $\gru(\gr{(1)}, \gr{(2)})$ to be a \emph{chordal} graph that
  contains all the vertices and edges in
  $\gr{(1)}$ and $\gr{(2)}$. We call such a graph, the
  computation graph of $\gr{(1)}$ and $\gr{(2)}$.
\end{definition}

Specifically, when $\alpha,\beta=0$, the joint distributions of $\dmp$
and $\dmq$ immediately decomposes the log term in the $\ab$-divergence
for this case, allowing for a complexity that is exponential to the
maximum treewidth between the 2 \glspl{dm}.

On the other hand, when either $\alpha$ or $\beta$ are non-zero, the
decomposition of the functional $\func$ from \cref{def:func}, and
therefore the $\ab$-divergence in this case, is not as
straightforward.  Substituting the joint distribution of $\dmp$ and
$\dmq$ into $\func$ eventually results in the following equation:
\begin{align}
  \begin{aligned}
    \func(
    {\pr},&{\qr}\;;\;g, h, g^{*}, h^{*})\\
    =& \sum_{\cli\in\clis(\graph_\pr)}\sum_{\xvals_{\cli}\in\dom_{\cli}}
       L\left(g^{*}\left[\pr_{\cli}^{\jt}\right](\xvals_{\cli})\right)
       \SP_{\cli}(\xvals_{\cli}) +\\
    &\sum_{\cli\in\clis(\graph_\qr)}\sum_{\xvals_{\cli}\in\dom_{\cli}}
      L\left(h^{*}\left[\qr_{\cli}^{\jt}\right](\xvals_{\cli})\right)
      \SP_{\cli}(\xvals_{\cli})
  \end{aligned}\label{eq:joint-decomposed}
\end{align}
where:
\begin{align}
  &\SP_{\cli}(\xvals_{\cli}) \label{eq:joint-sp}\\
  &=\sum_{\xvals \in \dom_{\xvar-\cli}}\left[
    \prod_{\cli \in \clis_{\pr}}
    g\left[\pr_{\cli}^{\jt}\right](\xvals_{\cli},\xvals)
    \right]
    \left[
    \prod_{\cli \in \clis_{\qr}}
    h\left[\qr_{\cli}^{\jt}\right](\xvals_{\cli},\xvals)
    \right]\nonumber
\end{align}
Therefore, computing $\func$ between two decomposable models relies on
the ability to obtain the sum-product $\SP_{\cli}$, from
\cref{eq:joint-sp}, for all maximal cliques, $\cli$, in the computation
graph $\gru(\dmp, \dmq)$. These sum-products can then be further
marginalized to obtain sum-products over maximal cliques of $\gr{\pr}$
and $\gr{\qr}$ in \cref{eq:joint-decomposed}.  Specifically,
\cite{lee2023} obtained these sum-products using belief propagation on
the junction tree of $\gru(\dmp,\dmq)$, with factors
$\{g[\pr_{\cli}^{\jt}] \mid \cli \in \clis(\gr{\pr})\} \cup
\{h[\qr_{\cli}^{\jt}] \mid \cli \in \clis(\gr{\qr})\}$.
\section{\glsentrytitlecase{comp}{long}s
  (\glsfmtshort{comp}s)}
\glsunset{comp}

For the purposes of developing the methods in this paper, we will
further abstract and generalize the problem of obtaining $\SP_{\cli}$
for all $\cli$ in $\clis(\grh)$ as a problem of obtaining
the sum-products over two sets of factors, $\fac{(1)}$ and
$\fac{(2)}$, defined over the maximal cliques over two chordal graphs,
$\gr{(1)}$ and $\gr{(2)}$, respectively.
\begin{equation}
  \label{eq:pgasp-sp}
  \forall \cli \in \clis(\grh) : 
  \SP_{\cli}(\xvals_{\cli}) =
  \sum_{\xvals\in\xdom_{\xvars- \cli}}
  \prod_{\phi\in\fac{(1)}\cup\fac{(2)}}
  \phi\left(\xvals_{\cli},\xvals\right)
\end{equation}
We can then further abstract the problem by observing that $\SP_{\cli}$
can be seen as a sum-product over the factors of a new chordal graph
constructed by merging the chordal \glspl{mn} $\mnf{(1)}$ and $\mnf{(2)}$
by taking their product according to \cref{def:mn-merge}.
\begin{definition}[Product of Markov netowrks]\label{def:mn-merge}
  Let $\mnf{(1)}$ and $\mnf{(2)}$ be two \glspl{mn}.  Then, taking the
  product of $\mn_{(1)}$ and $\mn_{(2)}$, $\mn_{(1)}\circ\mn_{(2)}$,
  involves:
  \begin{enumerate}
  \item first obtaining the chordal graph
    $\grh=\gru(\gr{(1)},\gr{(2)})$,
  \item then creating the factor set $\Psi=\fac{(1)}\cup\fac{(2)}$,
  \end{enumerate}
  resulting in a new chordal \gls{mn}, $\mn_{\grh}=(\grh,\Psi)$.
\end{definition}

Once $\mn_{\grh}$ is obtained, like \cite{lee2023} we can then
obtain $\forall \cli\in\clis(\grh) \: : \: \SP_{\cli}$ using belief
propagation on the clique tree of $\mn_{\grh}$ with initial factors
$\Psi$.
%
%
However, in order to aid future developments, we shall further extend
this result to obtain the sum-products over $m$ different factor sets,
and not just two.
\begin{lemma}[\glspl{comp}]\label{thm:mgasp}
  Let $\mnf{(1)}$, $\ldots$ , $\mnf{(m)}$ be $m$ \glspl{mn}, possibly
  defined over different sets of variables.
  Then, using the notion of products between \glspl{mn} from
  \cref{def:mn-merge}, we can obtain the following sum-products:
  \begin{equation}
    \label{eq:mgasp-sp}
    \forall \cli \in \clis(\grh) \: :\:
    \SP_{\cli}(\xvals_{\cli}) =
    \sum_{\xvals\in\xdom_{\xvars - \cli}}
    \prod_{\phi\in\bigcup_{i=1}^{m}\Phi_{(i)}}
    \phi\left(\xvals_{\cli},\xvals\right)
  \end{equation}
  by carrying out belief propagation on the junction tree of
  $\gru\bigl(\gr{(1)},\ldots,\gr{(m)}\bigr)$
  using the set of initial factors $\bigcup_{i=1}^{m}\Phi_{(i)}$.
\end{lemma}
\begin{proof}
  Using \cref{def:mn-merge}, we can first take the product of the $m$
  given \glspl{mn} to obtain the chordal \gls{mn}
  $\mn_{\grh}=(\grh,\Psi)$:
  $\mn_{\grh}=\mn_{(1)}\circ \ldots \circ \mn_{(m)}$ and
  $\Psi = \bigcup_{i=1}^{m}\Phi_{(i)}$. Then, recall that carrying out
  belief propagation on the junction tree of $\mn_{\grh}$ will result in
  the following beliefs over the maximal cliques of
  $\grh$~\cite[Corollary~10.2]{koller2009}:
  \begin{equation}
    \label{eq:bp-koller}
    \forall \cli\in\clis(\grh) \: : \:
    \beta_{\cli}(\xvals_{\cli})
    = \sum_{\xvals\in\xdom_{\xvars - \cli}}
    \prod_{\psi\in\Psi} \psi(\xvals,\xvals_{\cli})
  \end{equation}
  which is directly equivalent to the sum-products in \cref{eq:mgasp-sp}
  since $\Psi = \bigcup_{i=1}^{m}\Phi_{(i)}$.
\end{proof}

\section{Marginal Divergence}\label{sec:marg-div}
The marginal distribution over the set of variables $\zvars$ for a
given decomposable model $\dmp$ can be obtained by summing over the
domain of the variables not in $\zvars$,
$\bar{\zvars}=\xvars\setminus\zvars$.
\begin{equation}
  \label{eq:marg-dist}
  \pr_{\zvars}(\zvals)
  = \sum_{\bar{\zvals}\in\bar{\zdom}} \pr_{\xvals}(\bar{\zvals},\zvals)
  = \sum_{\bar{\zvals}\in\bar{\zdom}} \prod_{\cli\in\clis(\gr{\pr})}
  \prjt_{\cli}(\bar{\zvals},\zvals)
\end{equation}
However, in order to decompose the $\ab$-divergence between marginal
distributions of two \glspl{dm}, we need to find a factorization of
the sum over $\bar{\zdom}$ in these marginal distributions.
\glsreset{dm}
\subsection{Factorizing the marginal distribution
  of a \gls{dm}}\label{sec:marg-factorise}
Factorizing the marginal distribution in \cref{eq:marg-dist} will
require grouping together \glspl{cpt}, and
therefore maximal cliques of $\dmp$, that share any
variables in the set of variables to sum out, $\bar{\zvars}$. We shall call
such a grouping of maximal cliques in $\grP$, an $\sparts$-partition of
$\grP$. See \cref{fig:n-part-gr} for an example of such an
$\sparts$-partition.
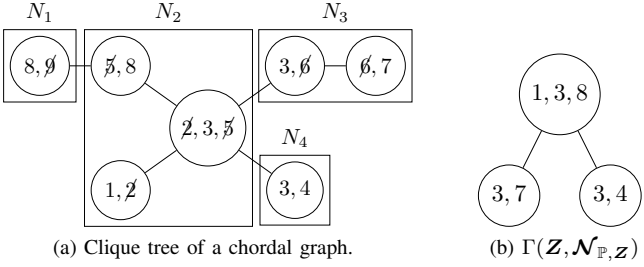
\begin{figure}
  \centering
  \subfloat[Clique tree of a chordal graph.]{  
    \centering \captionsetup{justification=centering}
    \begin{adjustbox}{width=0.3\textwidth}
      \begin{tikzpicture}[vert/.style={circle, draw=black}]
        \node[vert] (235) at (0,0) {$\cancel{2},3,\cancel{5}$};

        \node[vert] (12) at (-1.4,-1) {$1,\cancel{2}$};
        \node[vert] (34) at (1.4,-1) {$3,4$};
        \draw (235) -- (12);
        \draw (235) -- (34);

        \node[vert] (58) at (-1.4,1) {$\cancel{5},8$};
        \node[vert] (36) at (1.4,1) {$3,\cancel{6}$};
        \draw (235) -- (58);
        \draw (235) -- (36);

        \node[vert] (89) at (-2.7,1) {$8,\cancel{9}$};
        \node[vert] (67) at (2.7,1) {$\cancel{6},7$};
        \draw (58) -- (89);
        \draw (36) -- (67);

        \node[draw, inner sep=1mm,fit=(89) (89) (89) (89),
              label=above:{$\spart_{1}$}](n1){};
        \node[draw, inner sep=1mm,fit=(12) (235) (58) (12),
              label=above:{$\spart_{2}$}](n2){}; 
        \node[draw, inner sep=1mm,fit=(36) (36) (67) (36),
              label=above:{$\spart_{3}$}](n3){}; 
        \node[draw, inner sep=1mm,fit=(34) (34) (34) (34),
              label=above:{$\spart_{4}$}](n4){}; 
      \end{tikzpicture}       
    \end{adjustbox}\label{fig:n-part-gr}}
  \hfill
  \subfloat[$\sgr(\zvars,\sparts_{\pr,\zvars})$]{
    \centering \captionsetup{justification=centering}
    \begin{adjustbox}{width=0.12\textwidth}
      \begin{tikzpicture}[vert/.style={circle, draw=black}]
        \node[vert] (138) at (0,0) {$1,3,8$};
        \node[vert] (34) at (0.75,-1.5) {$3,4$};
        \node[vert] (37) at (-0.75,-1.5) {$3,7$};
        \draw (138) -- (34);
        \draw (138) -- (37);
      \end{tikzpicture}
    \end{adjustbox}    \label{fig:n-graph}
  }  
  \caption{An $\sparts$-partition over a clique tree for removing the
    variables striked-out, and its respective
    $\sparts$-graph.}\label{fig:n-part}
\end{figure}
\begin{definition}[$\sparts$-partition]\label{def:n-part}
  For any chordal graph $\grP$ and marginal variables $\zvars$, with
  $\bar{\zvars}=\xvars\setminus\zvars$, an $\sparts$-partition of $\grP$,
  $\sparts_{\pr,\zvars}$, is a set of vertex sets,
  $\forall \spart\in\sparts_{\pr,\zvars} : \spart\subset V(\grP)$,
  with the following properties:
  \begin{enumerate}
  \item
    $\{\clis(\grP(\spart)) \mid \spart \in
    \sparts_{\pr,\zvars}\}$ is a partition of the maximal cliques
    of $\grP$,
  \item
    $\{\bar{\zvars} \cap \xvars_{\spart} \mid \spart \in
    \sparts_{\pr,\zvars}, \bar{\zvars} \cap \xvars_{\spart} \neq
    \emptyset\}$ is a partition of $\bar{\zvars}$,
  \end{enumerate}
  One way to find an $\sparts$-partition of $\grP$ is to group
  maximal cliques of $\grP$ that share any variables that are also in
  $\bar{\zvars}$.
\end{definition}

\begin{corollary}[Decomposition of Marginal Distribution]
  Using $\sparts$-partitions from \cref{def:n-part}, the marginal
  distribution over variables $\zvars$ of a decomposable model $\dmp$
  can be factorized over the variables sets in $\sparts_{\pr,\zvars}$:
\begin{gather}
  \label{eq:marg-dist-decomp}
  \begin{aligned}
    \pr_{\zvars}(\zvals) &= \prod_{\spart\in\sparts_{\pr,\zvars}}
    \sum_{\bar{\zvals}\in\bar{\zdom}_{\spart}} \prod_{\cli\in\clis(\grP(\spart))}
    \prjt_{\cli}(\zvals_{\cli},\bar{\zvals}_{\cli})\\
    &=\prod_{\spart\in\sparts_{\pr,\zvars}}
    \partfac_{\spart\cap\zvars}[\pr](\zvals_{\spart})
  \end{aligned}
\end{gather}
where for ease of notation, we use $\partfac_{\sparts\cap{\zvars}}$ as
shorthand for marginalizing factors in $\spart$ to remove any variables
in $\bar{\zvars}$. See \cref{def:part-fac} for a full definition of
$\partfac_{\sparts\cap{\zvars}}$.
\end{corollary}
\begin{definition}[$(\spart\cap\zvars)$-Marginalised Factor,
  $\partfac_{\spart\cap\zvars}$]\label{def:part-fac}
  Through a slight abuse of notation, let $\spart\cap\zvars := \spart \cap V(\zvars)$.
  Then we define the $(\spart\cap\zvars)$-marginalized factor as
  such:
  \begin{equation}
    \label{eq:part-fac}
    \partfac_{\spart\cap\zvars}[\pr](\zvals_{\spart}) :=
    \sum_{\bar{\zvals}\in\bar{\zdom}_{\spart}} \prod_{\cli\in\clis(\grP(\spart))}
    \prjt_{\cli}(\zvals_{\cli},\bar{\zvals}_{\cli})
  \end{equation}
  Note that it is possible for $\partfac_{\spart\cap\zvars}$ to be
  defined over the empty set when
  $\spart\cap\zvars=\{\}\iff \xvars_{\spart}\subseteq\bar{\zvars}$. In
  such situations, $\partfac_{\spart\cap\zvars}$ is just a scalar
  factor,~i.e.\ a number in $\mathbb{R}$. For ease of exposition, we
  shall also define a function $\partfacs$ that returns a set of
  $(\spart\cap\zvars)$-marginalized factors given an $\sparts$-partition
  $\sparts_{\pr,\zvars}$:
  \begin{equation}
    \partfacs(\sparts_{\pr,\zvars}) :=
    \Bigl\{
    \partfac_{\spart\cap\zvars}[\pr] \Bigm|
    \spart \in \sparts_{\pr,\zvars}
    \Bigr\}
    \label{eq:def-part-facs}
  \end{equation}
\end{definition}

To summarize, we finally now have factorization of $\pr_{\zvars}$ in
terms of a product of factors defined over each set in the partition
$\sparts_{\pr,\zvars}$.
\begin{align}
  \pr_{\zvars}(\zvals)
  &=\prod_{\spart\in\sparts_{\pr,\zvars}}
    \partfac_{\spart\cap\zvars}[\pr](\zvals_{\spart})
\end{align}
These set of factors imply an new graph structure where an edge exists
between any two vertices in $\zvars$ if there is any set in
$\sparts_{\pr,\zvars}$ that contains both vertices together. We will
call these new graphs based on the sets in an $\sparts$-partition,
$\sparts$-graphs.  See \cref{fig:n-part-gr,fig:n-graph} for an example
of an $\sparts$-partition and its respective $\sparts$-graph.
\begin{definition}[Sets-to-Cliques graph constuctor, $\settocli(\sparts)$]
  \label{def:sets-to-clis}
  Let $\sparts$ be a set of vertex-sets. Then $\settocli$ is a graph
  constructor that creates a graph $\settocligr$ with the following
  vertices and edges:
  \begin{gather*}
    V(\settocligr) = \bigcup_{\spart\in\sparts} \spart\\
    E(\settocligr) = \bigl\{
    (u,v) \bigm| \spart\in\sparts , (v, u) \in \spart , v \neq u
    \bigr\}
  \end{gather*}
  Therefore, the graph constructor $\settocli(\sparts)$ ensures that
  each vertex-set in $\sparts$ is also as clique in the resulting
  graph $\settocligr$.
\end{definition}
\begin{definition}[$\sparts$-graphs,
  $\sgr(\sparts_{\pr,\zvars})$]\label{def:s-part-graph}
  Let $\grP$ be a chordal graph with an $\sparts$-partition
  $\sparts_{\pr,\zvars}$. Then $\sgr(\sparts_{\pr,\zvars})$ is a
  function that will take the graph constructed by
  $\settocli(\sparts_{\pr,\zvars})$ from \cref{def:sets-to-clis} and
  remove any vertex or edges associated with any variables in
  $\bar{\zvars}$.
\end{definition}
\begin{theorem}[The $\sparts$-graph created by $\sgr$ is a chordal
  graph]\label{thm:n-graph-chordal}
  Let $\grP$ be a chordal graph and $\zvars\subset\xvars$. Then the
  $\sparts$-graph, $\sgr(\sparts_{\pr,\zvars})$, is a chordal graph.
\end{theorem}
\begin{proof}
  Constructing $\sgr(\sparts_{\pr,\zvars})$ involves two steps: 1)
  merging maximal cliques in the same partition via
  $\settocli(\sparts_{\pr,\zvars})$, and 2) removing any vertices and
  edges associated with $\bar{\zvars}$. We will show that the resulting
  graph after both steps are chordal.
  \begin{itemize}
  \item Only adjacent maximal cliques in the clique tree of $\grP$ can
    be in the same partition in $\sparts_{\pr,\zvars}$. Merging adjacent
    maximal cliques in a clique tree will result in yet another clique
    tree, implying the graph $\settocli(\sparts_{\pr,\zvars})$ is also
    chordal.
  \item The deletion of all vertices and edges associated with
    $\bar{\zvars}$ is equivalent to the induced subgraph over the
    vertices $V(\zvars)$. But any induced subgraph of a chordal graph is
    also a chordal graph~\cite{blair1993}.
  \end{itemize}
  Therefore, $\sgr(\sparts_{\pr,\zvars})$ is a chordal graph.
\end{proof}

In conclusion, the factorization of $\pr_{\zvars}$ essentially creates a
new chordal Markov network consisting of the graph
$\sgr(\sparts_{\pr,\zvars})$ and factors
$\partfacs(\sparts_{\pr,\zvars})$,
$\mn_{\pr,\zvars}=\bigl(\sgr(\sparts_{\pr,\zvars}),\partfacs(\sparts_{\pr,\zvars})\bigr)$.
\subsection{Computing the Marginal
  \texorpdfstring{$\ab$}{ab}-Divergence between
  \Glspl{dm}}
With this factorization of the marginal distribution of a \gls{dm}, and
even a graphical representation of these factorizations, the problem of
computing the $\ab$-divergence between the marginal distributions of two
\glspl{dm} becomes surprisingly straightforward as a direct equivalence
can made to the problem of computing the joint divergence between two
\glspl{dm}.

Recall from \cref{sec:joint-div} that when \cite{lee2023} tackled the
problem of computing the joint divergence between $\dmp$ and $\dmq$,
they represented the joint distribution of a \gls{dm} as a
set of \glspl{cpt} defined over the maximal cliques of a chordal
graph. These \glspl{cpt} are fundamentally just ordinary factors over
their respective variables, and \cite{lee2023} did not rely on any
special property of these factors to compute the $\ab$-divergence
between these representations of \glspl{dm}.

Thanks to the factorization from \cref{sec:marg-factorise}, we have a
similar problem setup to computing the joint divergence. Specifically,
problem of computing the $\ab$-divergence between the marginal
distributions $\pr_{\zvars}$ and $\qr_{\zvars}$, of \glspl{dm}
$\dmp$ and $\dmq$ respectively, becomes a problem of computing the
$\ab$-divergence between the chordal \glspl{mn}:
$\mn_{\pr,\zvars}=\bigl(\sgr(\sparts_{\pr,\zvars}),\partfacs(\sparts_{\pr,\zvars})\bigr)$
and
$\mn_{\qr,\zvars}=\bigl(\sgr(\sparts_{\qr,\zvars}),\partfacs(\sparts_{\qr,\zvars})\bigr)$.

Therefore, the complexity of computing the $\ab$-divergence between
factorizations of $\pr_{\zvars}$ and $\qr_{\zvars}$ is just the
complexity of the approach in \cite{lee2023}. However, that does not
include the complexity of obtaining the factorization of $\pr_{\zvars}$
and $\qr_{\zvars}$.
\begin{theorem}\label{thm:n-part-cmplx}
  The worst case complexity of obtaining the factorization from
  \cref{sec:marg-factorise} of $\pr_{\zvars}$ for some \gls{dm} $\dmp$ is:
  \begin{equation}
    \label{eq:n-part-cmplx}
    \mathcal{O}\left(
      |\xvars| \cdot 2^{|\spart^{*}|}\right)
  \end{equation}
  where
  $ \spart^{*} := \argmax_{\spart\in\sparts_{\pr,\zvars}} |\spart|$.
\end{theorem}
\begin{proof}
  For any $\spart\in\sparts_{\pr,\zvars}$, obtaining the marginal factor
  $\partfac_{\spart\cap\zvars}[\pr]$ requires us to essentially iterate
  through each value $\xvals\in\xdom_{\spart}$ in the worst case,
  resulting in the complexity of $2^{|\spart|}$. This complexity is
  bounded by the complexity for the largest partition in
  $\sparts_{\pr,\zvars}$,
  $\spart^{*}:=\argmax_{\spart\in\sparts_{\pr,\zvars}}
  \abs{\spart}$. Since we need to obtain
  $\partfac_{\spart\cap\zvars}[\pr]$ for all
  $\spart\in\sparts_{\pr,\zvars}$, and $|\sparts_{\pr,\zvars}|$ in
  bounded by $|\xvars|$, the final complexity is
  $\mathcal{O}\left(|\xvars| \cdot 2^{|\spart^{*}|}\right)$.
\end{proof}

Therefore, the complexity of 
$D_{AB}^{(\alpha,\beta)}(\pr_{\zvars}||\qr_{\zvars})$ is:
\begin{align}
  \label{eq:comp-marg-cmplx}
  &D_{AB}^{(\alpha,\beta)}(\pr_{\zvars}||\qr_{\zvars})\\
  &\in
  \mathcal{O}\left(
    \abs{\xvars} \cdot 2^{|\spart^{*}|}\right)\cdot
  \begin{cases}
    \mathcal{O}(|\xvars|^{2} \cdot \twmax
    2^{\twmax+1})
    & \alpha,\beta=0\\
    \mathcal{O}(|\xvars|\cdot 2^{\tw(\grh)+1})
    & \text{otherwise}
  \end{cases}\nonumber
\end{align}
where
$\twmax= \max(\tw(\sgr(\sparts_{\pr,\zvars})),
\tw(\sgr(\sparts_{\pr,\zvars})))$ and $\grh=\gru(\sgr(\sparts_{\pr,\zvars}),\sgr(\sparts_{\pr,\zvars}))$.

\section{Conditional Divergence}
We will now show that having a factorization of any marginal
distribution for a \gls{dm} allows us to compute the conditional
$\ab$-divergence between two \glspl{dm}.
Recall from \cref{def:cond-div} the definition of the conditional
$\ab$-divergence we will use, where $\zvars\cup\yvars\subseteq\xvars$,
and $\zvars\cap\yvars=\{\}$:
\begin{gather*}
  D_{\AB}^{(\alpha,\beta)}(
  \pr_{\yvars \mid \zvars} \mid\mid \qr_{\yvars \mid \zvars})
  =\sum_{\zvals\in\zdom} \pr_{\zvars}(\zvals)
    D_{\AB}^{(\alpha,\beta)}(\prygz\mid\mid\qrygz)\nonumber
  \intertext{where}
  \begin{aligned}
    &D_{\AB}^{(\alpha,\beta)}(\prygz\mid\mid\qrygz)\\
    &=
    \begin{cases}
      \sum_{\yvals\in\ydom} \frac{1}{2}\left(\log \prygz(\yvals) -
      \log \qrygz(\yvals)\right)^{2}
      & \alpha,\beta=0\\
      \sum_{i} c_{i} \func(\prygz, \qrygz ;
      g_{i},h_{i},g^{*}_{i},h^{*}_{i})
      &\text{else}
    \end{cases}
  \end{aligned}
\end{gather*}
To assist in future developments, we shall define a notion of taking the
quotient between two \glspl{mn}.
\begin{definition}[Quotient of Markov netowrks]
  Assume we are given 2 \glspl{mn}, $\mnf{\xvars}$ and
  $\mnf{\zvars}$. Then the quotient $\mn_{\xvars} / \mn_{\zvars}$
  involves:
  \begin{enumerate}
  \item first finding chordal graph
    $\grh=\gru(\gr{\xvars},\gr{\zvars})$,
  \item then creating the factor set:
    $\Psi = \Phi_{\xvars} \cup \bigl\{1 / \phi \bigm| \phi \in
    \Phi_{\zvars} \bigr\} $
  \end{enumerate}
  This results in the chordal \gls{mn} $\mn=(\grh,\Psi)$.
\end{definition}
\begin{lemma}[Sum-Quotient]\label{thm:sum-quot}
  Let $\mnf{\xvars}$ and $\mnf{\zvars}$ be \glspl{mn} where
  $\zvars\subset\xvars$ and $\mn_{\zvars}$ is a marginalization of
  $\mn_{\xvars}$. Then the Sum-Quotients:
  \begin{equation}
    \label{eq:mgasp-sq}
    \forall \cli \in \clis(\grh) \: :\:
    \SP_{\cli}(\xvals_{\cli}) =
    \sum_{\xvals\in\xdom_{\xvars - \cli}}
    \frac
    {\prod_{\phi\in\Phi_{\xvars}} \phi\left(\xvals_{\cli},\xvals\right)}
    {\prod_{\phi\in\Phi_{\zvars}} \phi\left(\zvals_{\cli},\zvals\right)}
  \end{equation}
  can be obtained using belief propagation over the clique tree of the
  \gls{mn} resulting from the quotient $\mn_{\xvars}/\mn_{\zvars}$.
\end{lemma}
\begin{proof}
  The proof follows the exact same argument as the proof for
  \cref{thm:mgasp}. However, since quotients are involved, we need to
  ensure the quotient in \cref{eq:mgasp-sq} is always defined, assuming
  $0/0=0$~\cite[Definition~10.7]{koller2009}. We know that any sum, and
  therefore marginalization, over a product of factors that results in
  zero, implies that the original product of factors is zero as well.
  Therefore, since $\mn_{\zvars}$ is a marginalization of
  $\mn_{\xvars}$, whenever the denominator in \cref{eq:mgasp-sq} is $0$,
  its numerator is as well.
\end{proof}
\begin{corollary}[Decomposition of Conditional Distribution]
  \label{thm:decomp-cond}
  Let $\zvars,\yvars\subset\xvars$; $\zvars\cap\yvars=\emptyset$; and
  $\wvars=\yvars\cup\zvars$. Then the conditional distribution of a
  \gls{dm}, $\pr_{\yvars\mid\zvars}$, can be expressed as the quotient
  $\pr_{\wvars}/\pr_{\zvars}$. Both $\pr_{\wvars}$ and $\pr_{\zvars}$
  can be expressed as the chordal \glspl{mn}
  $\mn_{\wvars}=(\sgr(\sparts_{\pr,\wvars}),
  \partfacs(\sparts_{\pr,\wvars}))$ and
  $\mn_{\zvars}=(\sgr(\sparts_{\pr,\zvars}),
  \partfacs(\sparts_{\pr,\zvars}))$. The quotient
  $\mn_{\wvars}/\mn_{\zvars}$ then results in the \gls{mn}
  $\mn_{\yvars\mid\zvars}^{(\pr)}=(\grh,
  \partfac^{c}(\sparts_{\pr,\wvars},\sparts_{\pr,\zvars}))$ where:
  \begin{gather*}
    \grh = \gru(\sgr(\sparts_{\pr,\wvars}),\sgr(\sparts_{\pr,\zvars}))
    = \settocli(\sparts_{\pr,\zvars})\\
    \partfac^{c}(\sparts_{\pr,\wvars},\sparts_{\pr,\zvars}) :=
    \partfacs(\sparts_{\pr,\wvars}) \cup \bigl\{
    1/\partfac \bigm| \partfac \in \partfacs(\sparts_{\pr,\zvars})
    \bigr\}
  \end{gather*}
  and the product of the factors in $\mn_{\yvars\mid\zvars}^{(\pr)}$ is
  a decomposition of the conditional distribution
  $\pr_{\yvars\mid\zvars}$.
\end{corollary}

With \cref{thm:decomp-cond} we now have all the tools needed to tackle
computing $\conddiv$.



\begin{theorem}\label{thm:comp-cond-complx-else}
  Let $\zvars,\yvars\subset\xvars$; $\zvars\cap\yvars=\emptyset$; and
  $\wvars=\yvars\cup\zvars$. The complexity of computing the conditional
  $\ab$-divergence between two \glspl{dm}, $\dmp$ and $\dmq$,
  $\conddiv$, when either $\alpha$ or $\beta$ is $0$ is:
  \begin{align*}
    \conddiv \in \mathcal{O}\left(
    \abs{\xvars} \cdot 2^{\tw(\grh) + 1}\right)
  \end{align*}
  where
  $\grh=\gru(\settocli(\sparts_{\pr,\zvars}),
  \settocli(\sparts_{\qr,\zvars}))$.
\end{theorem}
\begin{proof}
  See Appendix~\ref{app:proof-cond-else}.
\end{proof}

\begin{theorem}\label{thm:comp-cond-complx-00}
  $\yvars\subset\xvars$; $\zvars\cap\yvars=\emptyset$; and
  $\wvars=\yvars\cup\zvars$.
  The complexity of $\conddiv$ when $\alpha,\beta=0$ is:
  \begin{align*}
    \conddivzero \in \mathcal{O}\left(
    \abs{\xvars}^{2} \cdot 2^{\tw(\grh) + 1}\right)
  \end{align*}
  where
  $\grh=\gru(
  \settocli(\sparts_{\pr,\zvars}),\settocli(\sparts_{\qr,\zvars}))$.
\end{theorem}
\begin{proof}
  See Appendix~\ref{app:proof-cond-00}.
\end{proof}

Therefore, the complexity of computing the conditional $\ab$-divergence
between two \glspl{dm} in the worst case is just the complexity when
both $\alpha$ and $\beta$ are $0$, as expressed in
\cref{thm:comp-cond-complx-00}.

\section{Experiments with QMNIST}\label{sec:qmnist}
We demonstrate the basic functionality of our method by analyzing
distributional changes within image data. Since image data is
interpretable by most humans, this will allow us to verify any findings
from the analysis of our method.

We consider the QMNIST dataset~\cite{yadav2019}. It was 
proposed as a reproducible recreation of the original MNIST
dataset~\cite{lecun1994,bottou1994}, using data from the NIST
Handprinted Forms and Characters Database, also known as NIST
Special Database 19~\cite{grother1995}.

An interesting fact about NIST Special Database 19 is that it consists of
handwriting from both NIST employees and high school students. In
QMNIST, images from NIST employees have the value $0$ and $1$ for the
variable \texttt{hsf}, while images from students have a value of
$4$. In \cref{fig:qmnist-mean}, by taking the mean over all the digits
written by each group, we can observe that the handwriting of NIST
employees tend to be more slanted compared to the
handwriting of the students. The question now is, can our method pick
up on these differences as well?
\begin{figure}
  \centering
  \subfloat[NIST employees]{
    \centering
    \begin{adjustbox}{width=0.22\textwidth}
      \includegraphics{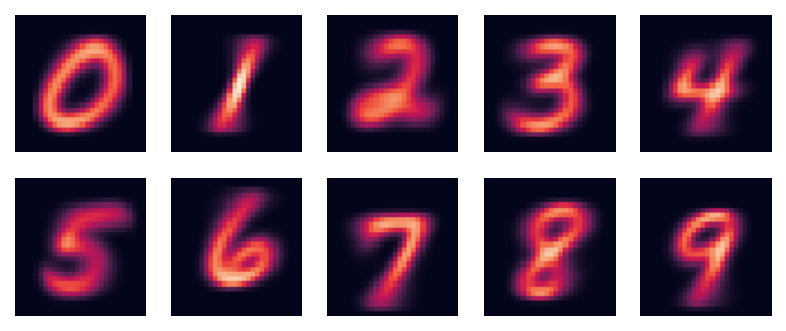}
    \end{adjustbox}
  }
  \hfill
  \subfloat[High school students]{
    \centering \captionsetup{justification=centering}
    \begin{adjustbox}{width=0.22\textwidth}
      \includegraphics{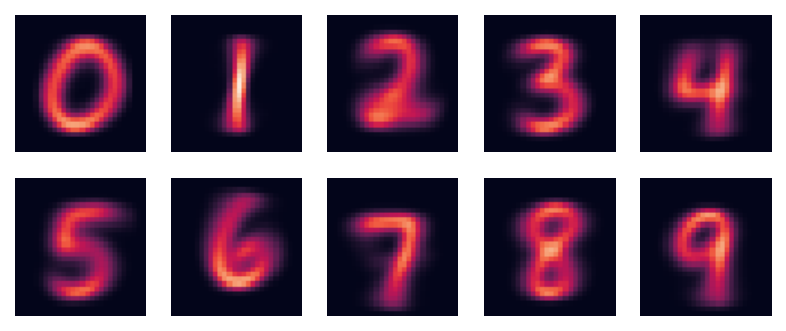}
    \end{adjustbox}
  }
  \caption{Mean over all digits written by (a) NIST employees, and (b)
    high school students}\label{fig:qmnist-mean}
\end{figure}
\begin{figure}
  \centering
  \includegraphics[width=0.48\textwidth]{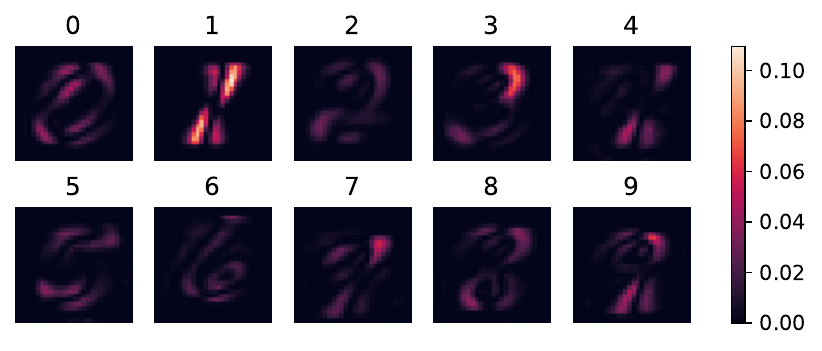}
  \caption{The marginal Hellinger distance over each pixel.}
  \label{fig:qmnist-12}
\end{figure}

In order to test this, we need to learn 10 pairs of \glspl{dm}, one for
each digit---i.e. 0, 1, etc.---with each pair containing the \glspl{dm}
learnt over the 2 writer groups---NIST employees and students. Since
each image has a resolution of $28\times28$ pixels, we have $784$
pixels, and therefore $784$ variables, to learn our \glspl{dm} over. We
do this by first learning the chordal graph structure for each
\glspl{dm} using
\emph{Chordalysis}~\cite{petitjean2013,petitjean2014,petitjean2015}.
Specifically from which, we employ the \emph{Stepwise Multiple Testing}
learner~\cite{webb2016} with a p-value threshold of $0.05$. Once we have
the structure for each \glspl{dm}, we learn their parameters using
\emph{pgmpy}~\cite{ankan2015}, specifically its
\texttt{BayesianEstimator} parameter estimator with Dirichlet priors and
a pseudo count of $1$. Once done, we now have 2 \glspl{dm} for each type
of digit.

By marginalizing the learnt \glspl{dm} over each of the $28\times 28$
pixels, and computing the marginal Hellinger distance between the
marginalized \glspl{dm} in each pair, we can immediately observe from
\cref{fig:qmnist-12} that the digit with the most pronounced difference
is $1$. Between the two writer groups, the top and bottom half of
the digit $1$ changes drastically, while the middle of the digit remains
relatively unchanged. This corresponds to the difference between how a
regular and slanted $1$ is written.

We can also observe some clear changes to the top right and bottom half
of digits $4$ and $9$, with these changes looking similar in both
digits. This corresponds to the observation from \cref{fig:qmnist-mean}
that both digits have a similar downward stroke on the right, which
changes similarly between the two writer groups. Other examples of
changes observable in \cref{fig:qmnist-mean} that are also highlighted
by our method in \cref{fig:qmnist-12} are: the digit $0$ changing
relatively more at the top and bottom left, $\{3,2,8\}$ mostly
changing mostly in the top right and bottom left, and $5$ changing
mostly at the top and bottom---but not as much in the middle.

However, our method seems to have picked up on more subtle changes that
that a brief visual inspection of \cref{fig:qmnist-mean} could have
missed.  For instance, although in \cref{fig:qmnist-mean} it might look
like digit $7$ differs the most in the downwards stroke at its bottom
half, \cref{fig:qmnist-12} indicates that the area with the greatest
difference for $7$ is in its top right instead.



\section{Quantum Error Analysis}

Our second set of experiments addresses the quantification of the
effective error behavior of quantum circuits.  A quantum circuit over
$n$ qubits is a $2^n \times 2^n$ unitary complex matrix $C$.  It acts on
a quantum state $\ket{\phi}$---a $2^n$-dimensional, $\ell_2$-normalized
complex vector.  The computation of an $n$-qubit quantum computer with
circuit $C$ can hence be written as $\ket{\psi}=C\ket{\phi}$.  Denoting
the $i$-th standard basis unit vector by $\boldsymbol{e}_i$, any quantum
state can be decomposed as
$\sum_{i=0}^{2^n-1} \alpha_i \boldsymbol{e}_i$ with
$\sum_{i=0}^{2^n-1} |\alpha_i|^2 = 1$. Clearly, the $2^n$-dimensional
output vector is never realized on a real-world quantum computer, as
this would require an exponential amount of memory.  Instead, a quantum
computer outputs a random $n$-bit string, where the bit string that
represents the standard binary encoding of the unsigned integer $i$ is
sampled with probability $\mathbb{P}(i)=|(C\ket{\phi})_i|^2$, according
to the Born-rule. However, due to various sources of error, samples from
real-world quantum computers are generated from a (sometimes heavily)
distorted distribution $\mathbb{Q}$.

In what follows, we show how to apply marginal divergences
$D_{\AB}^{(\alpha,\beta)}( \pr_{\zvars} \mid\mid \qr_{\zvars})$ for
quantifying the error for specific (subsets of) qubits. To this end, we
conduct two sets of experiments.

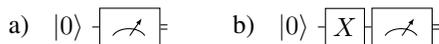
\begin{figure}
  \centering
  \centering
  \begin{tikzpicture}
    \node (0,0) {a)};
  \end{tikzpicture}
  \begin{tikzpicture}
    \begin{yquant}
      qubit {$\ket0$} q[1];
      measure q[0];
    \end{yquant}
  \end{tikzpicture}
  \hspace{0.5cm}
  \begin{tikzpicture}
    \node (0,0) {b)};
  \end{tikzpicture}
  \begin{tikzpicture}
  \begin{yquant}
    qubit {$\ket0$} q[1];
    x q[0];
    measure q[0];
  \end{yquant}
  \end{tikzpicture}
  \caption{Quantum circuits used in our single qubit experiments.}
  \label{fig:oneqcircuits}
\end{figure}

\begin{figure*}
  \includegraphics[width=0.32\textwidth]{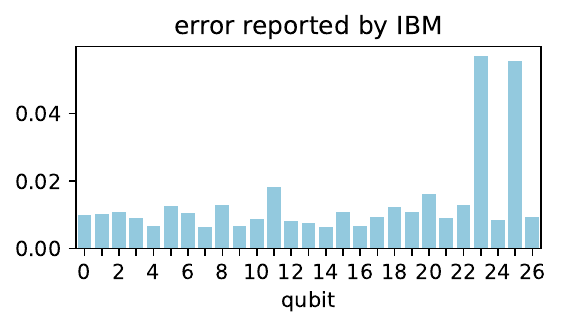}
  \includegraphics[width=0.32\textwidth]{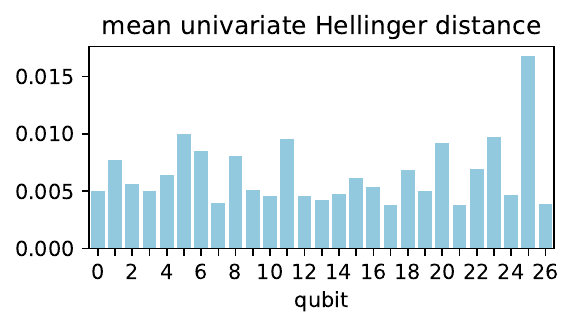}
  \includegraphics[width=0.32\textwidth]{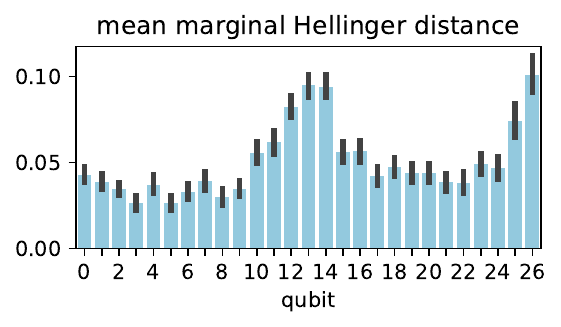}
  \caption{Left: Single qubit errors for \enquote{ibmq\_ehningen} as
    reported by IBM after we ran the experiments (June 1st,
    3am). Center: Mean Hellinger distance over the two circuits used in
    the single qubit experiment. Right: Mean univariate \emph{marginal}
    Hellinger distance over all 300 pairs of
    \glspl{dm}.}\label{fig:errors}
\end{figure*}

First, we consider simple single qubit circuits, shown in
\cref{fig:oneqcircuits}.  In \cref{fig:oneqcircuits} a), $C$ is
the identity matrix and the input state $\ket{\phi}$ is set to $\ket0$,
i.e., the zero state. This circuit is deterministic---an exact
simulation of this circuit will always output $0$. In
\cref{fig:oneqcircuits} b), $C=\begin{pmatrix}0&1\\1&0\end{pmatrix}$
is the so-called Pauli-$X$ matrix and the input state $\ket{\phi}$ is
set to $\ket0$. An exact simulation of this circuit will always output
$1$.

Nevertheless, on a real-world quantum computer, there is a small
probability, that both circuits will deliver the incorrect
result. Moreover, on quantum processors with more than one qubit, the
error depends on the specific hardware qubit that is selected to run the
circuit---each hardware qubit has its own error behavior. The precise
error is typically unknown, but quantum hardware vendors deliver rough
estimates of these quantities.

Here, we employ \enquote{ibmq\_ehningen}, an IBM Falcon r5.11
superconducting quantum processing unit (QPU) with 27 qubits.  The
corresponding single qubit errors that IBM reports for the time we ran
our experiments are provided in \cref{fig:errors} (left).  Wired
connections between qubits are shown in \cref{fig:coupling-map}.  We run
both circuits $N=10000$ times on each of the 27 qubits, resulting in a
total $2\times 27\times 10000 = 540000$ samples for the single qubit
experiment.

For the second set of experiments, we consider the quantum Fourier
transform (QFT) \cite{coppersmith2002}. QFT is an essential building
block of many quantum algorithms, notably Shor's algorithm for factoring
and computing the discrete logarithm, as well as quantum phase
estimation. We refer to \cite{NielsenChuang2016} for an elaborate
introduction to these topics.

Here, we run the QFT on $300$ uniformly random inputs on the same 27
qubit QPU as before. Each circuit is sampled $N=10000$ times, resulting
in a total of $2\times 10^6$ samples for the QFT experiment.

For the sake of reproducibility, the circuits, data, and python scripts
will be made available after acceptance of this manuscript.

\subsection{Single Qubit Divergence}\label{sec:quantum-single}
During transpilation\footnote{Transpilation is related to compilation: the qubits and 
operations used in a logical quantum circuit are mapped to the resources of some real-world
QPU architecture.}, a sub-set of all available
qubits is chosen to realize the logical circuit. The user can decide to consider only 
a specific sub-set of candidate qubits. 
Thus, reliable information about the
effective error behavior of each qubit when used within a quantum
circuit is of utmost importance.

Based on data from the single qubit experiments, we estimate
$\mathbb{Q}_i$, that is, the probability that the $i$-th qubit takes a
specific value (0 or 1). Moreover, we obtain an empirical estimate of
the ground truth $\mathbb{P}_i$ from the exact simulated results, which
allows us to compute the Hellinger distance between $\mathbb{P}$ and
$\mathbb{Q}$ for each qubit.  We do this for both circuits from
\cref{fig:oneqcircuits}. The average qubit-wise Hellinger distance is
reported in \cref{fig:errors} (center).

The results help us to quantify the \underline{isolated} error behavior
of each qubit. We see that our results qualitatively agree with the
errors reported by IBM, e.g., the top-5 qubits which show the largest
divergence $\{5,11,20,23,25\}$ also accumulate the largest error mass as
reported by IBM. Explaining all the discrepancies is neither
possible---since the exact procedure for reproducing the IBM
errors in unknown---nor required. The result clearly shows that the
numbers reported by IBM quantify the \underline{isolated} error behavior
of each qubit. We will now see that one obtains a different picture when
qubits are used within a larger circuit (which is the most common
use-case).

To this end, we consider the results from the QFT circuits. We learn 300
pairs of \glspl{dm} on the samples using the same methodology as in
\cref{sec:qmnist}. We then compute the marginal divergence for each
qubit and report the average over all 300 QFT circuits in
\cref{fig:errors} (right).  As these numbers are marginals of the full
joint divergence, we now quantify how each qubit behaves on average when
it is part of a larger circuit.

We see that the result mostly disagrees with both, the errors reported
by IBM and the single qubit divergences reported by us. First, we see
that qubits 12, 13, and 14---the qubits directly in the middle of the
coupling map shown in \cref{fig:coupling-map}--which are qubits that,
individually, did not have much error according to both our experiments
in \cref{fig:errors} (center) and the reported errors from IBM. However,
it is not a coincidence that most of these high error qubits are
concentrated in the middle of the circuit. Unlike in an idealized QPU,
the qubits in the physical QPU we used are not all coupled with each
other. Therefore, routing has to take place whenever a quantum circuit
operates on pairs of qubits which are not physically connected. As a
result, the states of the qubits in the middle of the circuit---qubits
12, 13, and 14---needs to be pushed around and stored in other qubits
more often. This is done in order to make way for routes between qubits
of opposing sides of the circuit. The more frequently these states get
pushed around, the greater the chance of an error occurring, hence why
these qubits in the middle exhibit the greatest amount of error.  To the 
best of our knowledge no quantum hardware vendor reports such numbers,
although they are highly important when logic qubits are transpiled to a 
circuit that is mapped to the actual hardware.

From \cref{fig:errors} (right), we can also observe a spike in the
Hellinger distance for qubits 25 and 26.  While qubit 25 is a heavy
hitter in \cref{fig:errors}, this does not apply for qubit 26. Moreover,
qubit 26 is a ``leaf'' qubit (w.r.t \cref{fig:coupling-map}) and hence
not used heavily for routing. Therefore, our approach reveals a 
type of error that is not detected by standard methods, e.g., cross-talk
or three-level-system activity.

\subsection{High-Order Divergence}\label{sec:quantum-multi}

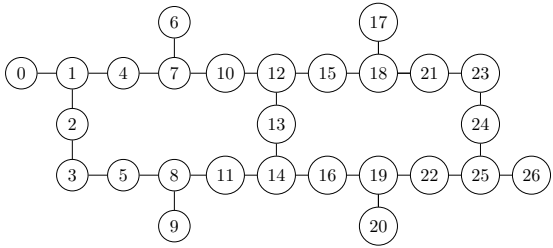
\begin{figure}
  \centering
      \begin{adjustbox}{width=0.4\textwidth}
      \begin{tikzpicture}[vert/.style={circle, draw=black}]
        \node[vert] (0) at (0,3) {$0$};
        \node[vert] (1) at (1,3) {$1$};
        \node[vert] (2) at (1,2) {$2$};
        \node[vert] (3) at (1,1) {$3$};

        \node[vert] (4) at (2,3) {$4$};
        \node[vert] (5) at (2,1) {$5$};

        \node[vert] (6) at (3,4) {$6$};
        \node[vert] (7) at (3,3) {$7$};
        \node[vert] (8) at (3,1) {$8$};
        \node[vert] (9) at (3,0) {$9$};

        \node[vert] (10) at (4,3) {$10$};
        \node[vert] (11) at (4,1) {$11$};

        \node[vert] (12) at (5,3) {$12$};
        \node[vert] (13) at (5,2) {$13$};
        \node[vert] (14) at (5,1) {$14$};

        \node[vert] (15) at (6,3) {$15$};
        \node[vert] (16) at (6,1) {$16$};

        \node[vert] (17) at (7,4) {$17$};
        \node[vert] (18) at (7,3) {$18$};
        \node[vert] (19) at (7,1) {$19$};
        \node[vert] (20) at (7,0) {$20$};

        \node[vert] (21) at (8,3) {$21$};
        \node[vert] (22) at (8,1) {$22$};

        \node[vert] (23) at (9,3) {$23$};
        \node[vert] (24) at (9,2) {$24$};
        \node[vert] (25) at (9,1) {$25$};

        \node[vert] (26) at (10,1) {$26$};

        \draw (0) -- (1);
        \draw (2) -- (1);
        \draw (2) -- (3);

        \draw (1) -- (4);
        \draw (3) -- (5);

        \draw (4) -- (7);
        \draw (5) -- (8);
        \draw (7) -- (6);
        \draw (8) -- (9);

        \draw (7) -- (10);
        \draw (8) -- (11);

        \draw (10) -- (12);
        \draw (11) -- (14);
        \draw (12) -- (13);
        \draw (14) -- (13);

        \draw (12) -- (15);
        \draw (14) -- (16);

        \draw (15) -- (18);
        \draw (16) -- (19);
        \draw (18) -- (17);
        \draw (19) -- (20);

        \draw (18) -- (21);
        \draw (19) -- (20);
        \draw (18) -- (17);
        \draw (19) -- (20);

        \draw (18) -- (21);
        \draw (19) -- (22);

        \draw (21) -- (23);
        \draw (22) -- (25);
        \draw (23) -- (24);
        \draw (25) -- (24);

        \draw (25) -- (26);
      \end{tikzpicture}
    \end{adjustbox}
    \caption{The coupling map for the 27 qubit IBM Falcon r5.11 QPU}
    \label{fig:coupling-map}
\end{figure}

\begin{figure}
\centering
  \includegraphics[width=0.49\textwidth]{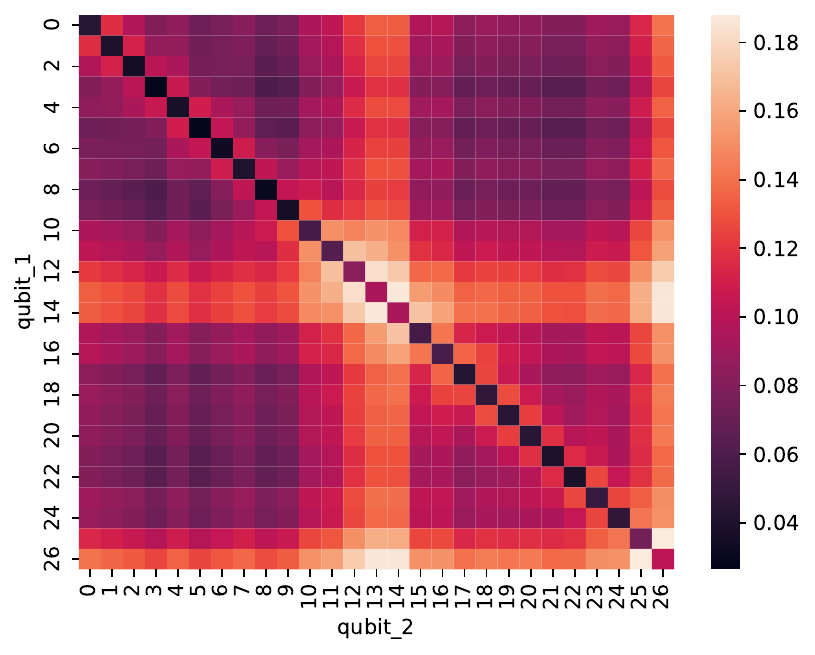}
  \caption{Mean pairwise marginal Hellinger distance over all pairs of
    \glspl{dm}.}\label{fig:marg-2-mean}
\end{figure}

We finish our analysis by investigating marginal divergence beyond
single qubits. \cref{fig:marg-2-mean} shows the result of computing the
pairwise marginal Hellinger distance between the \glspl{dm} in each
pair, taking the mean over all the pairs. The resulting heatmap is
symmetrical due to Hellinger distance being a symmetric divergence. The
first immediate observation is the bright bands over pairs that include
qubits that already have a high univariate marginal Hellinger
distance---such as qubits 12, 13, 14, 25, and 26. Another interesting
observation is the bright band that occurs across the diagonal of the
heatmap. We believe this is due to higher order interactions between
coupled qubits, that do not necessarily exist between qubits that are
multiple jumps away from each other. There are multiple causes for these
higher-order interactions, ranging from unintended physical
processes---such as cross-talk~\cite{ketterer2023,perrin2023}---to
limitations of current QPUs---such as states being swapped between
qubits solely for the purposes of routing.

These higher-order interactions are not just limited to pairwise
interactions. As an example, we have discovered 70 cases where a 3-tuple
of qubits, $A$, has a greater Hellinger distance than some other 3-tuple
of qubits, $B$, despite the maximum Hellinger distance over all the
pairwise combinations of the qubits in $A$ being less than the minimum
over that of $B$.  This clearly shows that our method is applicable for
recovering high-order error behavior that cannot be deduced by simply
investigating low-order terms.

\section{Conclusion}
In this work, we explained how to compute the marginal and conditional
$\ab$-divergence between any two \glspl{dm}. In the process, we showed
how to decompose the marginal distribution of a \gls{dm} over any subset
of variables, and how this can be used
to decompose the conditional distribution of a \gls{dm} as well.  In
order to compute the marginal and conditional $\ab$-divergence based on
these decompositions, we introduced a notion of the product and quotient
between \glspl{mn}. 

An initial numerical experiment on the image benchmark dataset
QMNIST~\cite{yadav2019} showed that our methods can be applied for
analyzing the differences in handwriting between NIST employees and high
school students.

Finally, we proposed a completely novel error analysis of contemporary
quantum computers, based on marginal divergences. To this end, we
collected more than 2.5 million samples from a 27 qubit
superconducting quantum processor. 
From this data, we found that marginal divergences can yield useful
insights into the effective error behavior of physical qubits. This
procedure can be applied to every quantum computing device and is of
utmost importance for selecting which hardware qubits will eventually
run a quantum algorithm, e.g., for feature selection or probabilistic
inference \cite{mucke23,mucke22,piatkowski2022quantum}.


A restriction of our approach is the limitation to computing the
$\ab$-divergence between \glspl{dm}. Although
this family of divergences is fairly broad, it does not include some
well-known divergences, such as the Jensen-Shannon
divergence~\cite{lin1991}. It might be possible to obtain an
approximation of these divergences by using our approach to compute a
Taylor expansion of these divergences instead. However, more work is
needed to figure out if this is even a viable path forward.

Nevertheless, our work paves the way for the principled analysis of
high-order relations within high-dimensional models. An interesting
direction that arises from our insights is the question which relations
one should analyze, since their number grows combinatorially with
respect to the size of the subsets we wish to analyze. Therefore more
research is urgently needed to develop advanced tools that can aid in
the analysis of higher dimensional subsets.

\appendices
\section{Proof for \cref{thm:comp-cond-complx-else}}\label{app:proof-cond-else}
  Using the product-quotient of \glspl{mn}, we can re-express the
  factors in the log-term of the conditional $\ab$-divergence:
  \begin{align*}
    &\conddiv\\
    &=\sum_{\zvals\in\zdom} \pr_{\zvars}(\zvals)
      \sum_{\yvals\in\ydom}
      g[\pr_{\yvars\mid\zvars}](\yvals\mid\zvals)
      h[\qr_{\yvars\mid\zvars}](\yvals\mid\zvals)\\
    &\qquad\qquad\qquad\quad \log \bigl(
      g^{*}[\pr_{\yvars\mid\zvars}](\yvals\mid\zvals)
      h^{*}[\qr_{\yvars\mid\zvars}](\yvals\mid\zvals)
      \bigr)
  \end{align*}
  as factors of the \gls{mn}, $\mn_{\log}=(\grh_{\log}, \Psi_{\log})$,
  where:
  $\grh_{\log} =\gru(\settocli(\sgr(\sparts_{\pr,\zvars}),
  \settocli(\sgr(\sparts_{\qr,\zvars}))$ and
  $\Psi_{\log} = \bigl\{g^{*}[\partfac] \bigm| \partfac \in
  \partfacs^{c}(\sparts_{\pr, \wvars}, \sparts_{\pr, \zvars}) \bigr\}
  \cup \bigl\{h^{*}[\partfac] \bigm| \partfac \in
  \partfacs^{c}(\sparts_{\qr, \wvars}, \sparts_{\qr, \zvars}) \bigr\}$.
  Therefore, the conditional $\ab$-divergence in this case can be
  expressed as such:
  \begin{align*}
    \conddiv
    &=\sum_{\psi\in\Psi}
      \sum_{\wvals_{L}\in\wdom_{L}}
      \log \psi(\wvals_{L})
      \SPQ_{L}(\wvals_{L})
  \end{align*}
  where $L:=V(\psi)$ and
  \begin{align*}
    \SPQ_{L}(\wvals_{L}) = \sum_{\wvals\in\wdom_{\wvars - L}}
    \pr_{\zvars}(\zvals) 
    {g[\pr_{\yvars\mid\zvars}](\yvals,\zvals)}
    {h[\qr_{\yvars\mid\zvars}](\yvals,\zvals)}
  \end{align*}
  $\SPQ_{L}$ is then a sum-product-quotient that can be obtained by
  carrying out belief propagation on the clique tree of
  $\grh=\gru(\settocli(\sparts_{\pr,\zvars}),
  \settocli(\sparts_{\qr,\zvars}), \sgr(\sparts_{\pr,\zvars}))
  =
  \gru(\settocli(\sparts_{\pr,\zvars}),
  \settocli(\sparts_{\qr,\zvars}))$ with factors
  \begin{align*}
    \Psi =\;
    &\partfacs(\sparts_{\pr,\zvars}) \cup \bigl\{
      g[\partfac] \bigm| \partfac \in
      \partfacs^{c}(\sparts_{\pr, \wvars}, \sparts_{\pr, \zvars})
      \bigr\} \cup\\
    &\quad\bigl\{
      h[\partfac] \bigm| \partfac \in
      \partfacs^{c}(\sparts_{\qr, \wvars}, \sparts_{\qr, \zvars})
      \bigr\}
  \end{align*}
  which has the complexity of $\bigO(2^{\tw(\grh)+1})$. Since we need to
  carry out this process for each factor in $\Psi_{\log}$, and
  $|\Psi_{\log}|\in\bigO(|\xvars|)$, then the final complexity of
  computing $\conddiv$ when either $\alpha$ or $\beta$ is 0 is
  $\bigO(|\xvars|\cdot 2^{\tw(\grh)+1})$.

\section{Proof for \cref{thm:comp-cond-complx-00}}\label{app:proof-cond-00}
  Recall the conditional $\ab$-divergence when $\alpha,\beta=0$:
  \begin{align*}
    &\conddivzero\\
    &=\sum_{\zvals\in\zdom} \pr_{\zvars}(\zvals)
      \sum_{\yvals\in\ydom} \frac{1}{2}
      \left(\log \prygZ(\yvals\mid\zvals) -
      \log \qrygZ(\yvals\mid\zvals)\right)^{2}
  \end{align*}
  We can then express the factors resulting from the division between
  the two conditional distributions as the factors of the \gls{mn}
  $\mn_{\log}=(\gr{},\Phi)$ where:  $\gr{}=\gru(
  \settocli(\sparts_{\pr,\zvars}),\settocli(\sparts_{\qr,\zvars}))$ and
  $\Phi =
  \partfacs^{c}(\sparts_{\pr,\wvars},\sparts_{\pr,\zvars}) \cup \{
  1 / \partfac \mid \partfac \in
  \partfacs^{c}(\sparts_{\qr,\wvars},\sparts_{\qr,\zvars})
  \}$. Therefore:
  \begin{align*}
    &\conddivzero
    =\sum_{\zvals\in\zdom} \frac{\pr_{\zvars}(\zvals)}{2}
      \sum_{\yvals\in\ydom} \Biggl(\sum_{\phi\in\Phi}
      \log \phi(\yvals,\zvals)
      \Biggr)^{2}\\
    &=\frac{1}{2}\sum_{\phi^{+}\in\Phi}
      \sum_{\phi^{*}\in\Phi}\sum_{\zvals\in\zdom}\sum_{\yvals\in\ydom}
      \pr_{\zvars}(\zvals) \log \phi^{+}(\yvals,\zvals)
      \log \phi^{*}(\yvals,\zvals)
  \end{align*}
  The sum-product over $\zdom$ and $\ydom$, for each
  $(\phi^{+},\phi^{*})\in\Phi\times\Phi$, can then be expressed as a
  sum-product over the factors of the \gls{mn} $\mn=(\grh,\Psi)$, where:
  $\grh = \gru(\sgr(\sparts_{\pr,\zvars}),
  \gr{}(V(\phi^{+})), \gr{}(V(\phi^{*})))$, and
  $\Psi=\{\log\phi^{*}, \log\phi^{+}\} \cup
  \partfacs(\sparts_{\pr,\zvars})$.

  In order to find the complexity of computing this sum product, we need
  to find a bound on the treewidth of $\grh$. We know that the treewidth
  of the induced subgraphs $\gr{}(V(\phi^{*}))$ and $\gr{}(V(\phi^{+}))$
  is bounded by the treewidth of $\gr{}$. Furthermore, since
  $E(\sgr(\sparts_{\pr,\zvars}))\subseteq E(\gr{})$, then
  $\tw(\sgr(\sparts_{\pr,\zvars})) \leq \tw(\gr{})$. Therefore the
  complexity of computing the sum-product over $\Psi$ is exponential
  with respect to $\tw(\gr{})$.

  Since we have to obtain this sum-product for each
  $(\phi^{+},\phi^{*})\in\Phi\times\Phi$, and
  $\abs{\Phi\times\Phi}=\abs{\Phi}^{2}\leq\abs{\xvars}^{2}$, the final
  complexity is:
  $\bigO(\conddivzero)\in\bigO(\abs{\xvars}^{2}2^{\tw(\gr{})+1})$
  

\section*{Acknowledgment}
This work was supported by the Australian Research Council award
DP210100072 as well as the Australian Government Research Training
Program (RTP) Scholarship. Additionally, this work has been funded by
the Federal Ministry of Education \& Research of Germany and the state
of North-Rhine Westphalia as part of the Lamarr-Institute for Machine
Learning and Artificial Intelligence.

\bibliographystyle{IEEEtran}

\end{document}